\documentclass[times,sort&compress,3p]{elsarticle}
\journal{Journal of Big Data Research}
\usepackage[labelfont=bf]{caption}

\usepackage{amsmath,amsfonts,amssymb,amsthm,booktabs,color,epsfig,graphicx,hyperref,url}

\theoremstyle{plain}
\newtheorem{theorem}{Theorem}

\newtheorem{lemma}{Lemma}

\theoremstyle{definition}

\begin{document}

\begin{frontmatter}

\title{Random Manifold Sampling and Joint Sparse Regularization for Multi-label Feature Selection}

\author[1]{Haibao Li \corref{mycorrespondingauthor}}
\author[2]{Hongzhi Zhai}

\address[1]{College of Sciences,Northeastern University,Shenyang,110819,China}
\address[2]{Business School,Shandong University,Weihai,264209,China}

\cortext[mycorrespondingauthor]{Corresponding author. Email address: \url{lihaibao98@163.com.}}

\begin{abstract}
  Multi-label learning is usually used to mine the correlation 
between features and labels, 
and feature selection can retain as much information as possible 
through a small number of features.
$\ell_{2,1}$ regularization method can get sparse coefficient matrix, 
but it can not solve multicollinearity problem effectively. 
The model proposed in this paper can obtain the most relevant few features 
by solving the joint constrained optimization problems of $\ell_{2,1}$ 
and $\ell_{F}$ regularization.
In manifold regularization, we implement random walk strategy based on 
joint information matrix, and get a highly robust neighborhood graph.
In addition, we given the algorithm for solving the model and proved 
its convergence.
Comparative experiments on real-world data sets show that 
the proposed method outperforms other methods.
\end{abstract}


\end{frontmatter}

\section{Introduction\label{sec:1}}

Different from multi-classification problem, 
an instance can have multiple labels at the same time in multi-label learning.
As Zhihua Zhou and Zhiling Cai et al. studied 
in \cite{Zhang2007-ML-KNN,Zhang2014-LIFT,Cai2018-MSSL,Weng2018}, 
one input data will correspond to multiple outputs.
It is facing some problems that need to be solved urgently, 
although multi-label learning has been very common in production and life.
The first problem is that we often assume that 
labels are independent of each other when we make statistical inference, 
but in fact there is a strong correlation between different labels.
The second problem is that multi-label data usually 
has very high feature dimension and label dimension, 
which leads to a sharp increase in learning cost and labeling cost.
The third problem is that the number of samples on each label varies greatly. 
This problem is called label imbalance, 
which will lead to the failure of some labels, 
resulting in waste of information \cite{Tahir2012,Zhang2020}.

When using the information of labels, 
previous studies often explore the high-order association between labels. 
Based on the association between labels, 
these methods can be roughly divided into three categories:
(1) The first-order method, 
which always assumes that labels are independent of each other \cite{Dembczynski2012}.
(2) Second-order method, 
which considers the correlation between labels \cite{Weng2018}.
(3) There are also high-order methods. 
Although they take more account of the interaction between labels, 
they are seldom used in practice because of their complexity \cite{Kang2006-CLP}.

If the dimension of data is too high, 
it will lead to "dimension disaster" \cite{Donoho2000, Johnstone2009}. 
Therefore, in order to avoid over-fitting in the process of model training, 
dimension reduction is usually carried out in the data preprocessing stage.
Dimension reduction methods can be roughly divided into two categories: 
feature extraction and feature selection.
The former is generally an unsupervised learning method, 
which obtains a few new features by combining, transforming and 
spectral decomposition of the original features.
These new features greatly retain the information of the original features, 
but lost the meaning of the original features.
For example, a small number of new features obtained by PCA 
after linear combination of original features are independent of each other, 
but the new features are not interpretable.
Through eigenmap and neighborhood graph, 
LPP \cite{He2004-LPP} preserves the local structure of data 
while reducing dimensions.
Because of the poor interpretability of feature extraction methods, 
they are usually only used as intermediate results or visualization.
Feature selection method directly selects 
a small number of high-value features from the original features 
by some means to achieve dimension reduction.
Feature selection methods can generally be divided into three types: 
filter, wrapper and embedding. 
In the first two methods, 
the process of feature selection and model training is separate, 
so the model does not work very well on the task, 
and the computational overhead is relatively high.
However, embedding method automatically completes feature selection 
in the process of training model, 
and its performance and computational overhead are usually superior to 
the former two. 
MCFS \cite{Cai2010-MCFS}, RFS \cite{Nie2010-RFS}, 
LASSO \cite{Tibshirani1996-LASSO} and 
other embedding methods generate sparse matrices for feature selection. 
They have been widely used in various tasks and have good performance.

The performance of global-based feature selection method 
on ultra high dimensional data is not always satisfactory, 
because the distribution of data in high-dimensional space is 
usually very complex.
By paying attention to the local structure of data, 
manifold learning method \cite{Cai2018-MSSL, Zhang2019-MDFS, He2004-LPP} is 
superior to other methods.
Nevertheless, manifold learning can't avoid the problem of "short circuit", 
especially when the number of neighbors is not suitable.

When we consider the correlation between features, 
Hyunki Lim \cite{LimH2021} builds a model to analyze 
the correlation between label pairs, 
and obtains a subset of features with low correlation.
Arthur and Robert et al. proposed \cite{Hoerl1970-Ridge} ridge regression, 
which obtained a subset of variables with low correlation by 
punishing highly correlated variables.
However, even when variables are highly correlated, 
they may get better results when working together than 
when working with several independent variables.
The elastic net method \cite{Zou2005-E-net} combines 
two or more features of a set of related features for consideration, 
so that the model has the ability to evaluate 
the collaborative work of features and 
avoids breaking the situation of multi-feature collaborative work.

Since the predicted value on labeled data should be greater than 
the predicted value on unlabeled data, 
we take the sum of squares of prediction errors as the loss function.
In particular, when considering the neighborhood graph of manifold sampling, 
we use the joint information and 
implement the multi-step random walk strategy to 
construct the graph Laplacian matrix, 
so that the model can retain a highly robust local structure.
On the basis of the above, 
we combine $\ell_{F}$ and $\ell_{2,1}$ regularization sum to 
obtain a highly sparse coefficient matrix, 
so as to realize feature selection.

The rest of this article is organized as follows:
The related work is introduced in Section 2. 
In section 3, 
the model is established and the expression of the model is derived. 
In section 4, 
the algorithm for solving the model is given and 
the convergence of the algorithm is proved. 
In Section 5, 
data experiments are designed and the performance of 
the proposed algorithm on different data sets is demonstrated. 
Section 6 summarizes the work of this paper.

\section{Related works\label{sec:2}}

Predecessors have done a lot of research on 
how to make full use of label information.
In ML-KNN, Zhihua Zhou et al. \cite{Zhang2007-ML-KNN} predicted 
the label of new samples through the information of neighboring samples and 
posterior probability.
Feng Kang and Wei Weng et al. make full use of 
the relevant information between labels through 
the label propagation algorithm \cite{Kang2006-CLP, Weng2018, Wang2013}, 
and get considerable results.
When measuring the similarity between instances, 
Hamers and Kosub et al. \cite{Hamers1989, Kosub2019} suggest 
using Jaccard index. 
In their research, the author gives the probability basis of Jaccard index 
and shows its excellent performance.
However, the above methods can not make full use of the sample information.
For example, 
instance1:$x_1=(2,1,0,1), y_1=(0,1,1)$, 
instance2:$x_2=(2,1,1,1), y_2=(1,0,0)$, 
instance3:$x_3=(2,0,0,1), y_3=(0,0,1)$,
instance1 is more similar to instance2 if we only consider features, 
but if we measure similarity by labels, 
instance1 and instance3 are more similar.
Balasubramanian \cite{Balasubramanian2002-Isomap} 
called this phenomenon "short circuit", 
and based on this, put forward the concept of manifold learning.
n order to avoid the "short circuit" phenomenon, 
Roweis proposed LLE method in \cite{Roweis2000-LLE}, 
which approximates the global nonlinear structure by local linear embedding.
LPP method \cite{He2004-LPP} firstly points out that 
high-dimensional data is projected into low-dimensional space, 
so as to realize dimension reduction by manifold learning method. 
A large number of studies have verified the effectiveness of LPP method.
Quanmao Lu et al. \cite{Lu2018} applied manifold embedding method to 
unsupervised learning dimension reduction and obtained excellent results.
Miao Qi and Ronghua Shang et al. \cite{Shang2020, Qi2018} 
put forward the theory of projecting the loss function to 
the low-dimensional subspace and obtaining the optimal subspace 
by matrix factorization.
However, these manifold methods are all based on KNN to 
construct neighborhood graphs, 
so the stability of local structure is easily affected by 
the number of neighbors.
Laurensvan der Maaten and Xiaokai Wei et al. \cite{Van2008-t-SNE, Wei2016} 
greatly improves the stability of manifold structure by 
adopting stochastic neighborhood embedding (SNE) method.
Different from punishing the similarity between neighbors, 
Aiping Huang et al. \cite{Huang2021} pay more attention to 
the linear combination of similarity vectors of samples, 
and they take the $\ell_{F}$-norm of similarity residual matrix 
as the loss function.
All the above studies are unsupervised learning methods, 
so they can't make effective use of label information when labels are given.
We noticed that Ronghua Li and Fatemeh Vahedian et al. 
\cite{Li2015, Vahedian2017,Li2019} studied the excellent properties of 
random walk in extracting the similarity of instances. 
Inspired by these studies, we introduce random walk strategy into 
the similarity sampling of manifold structure, 
which makes the model more robust.

RFS method \cite{Nie2010-RFS} adds parameter regularization term to 
the multi-label learning model, 
which makes the coefficient matrix sparse and realizes feature selection. 
Similar to $\ell_{1}$-norm, $\ell_{2,1}$-norm does not consider the 
synergy of multiple features, 
and it tends to choose one of a group of highly related features.
Ridge regression \cite{Hoerl1970-Ridge} abandons the unbiasedness of 
estimated parameters, which makes the selected features have low correlation.
In contrast, the elastic net \cite{Zou2005-E-net} method weighs 
the effects of multiple features more comprehensively, 
thus training a more capable model.
Recently, Mohammad Ghasem Akbari and Gholamreza Hesamian \cite{Akbari2019} 
proposed a semiparametric model, 
which applies kernel smoothing and elastic penalty methods to 
fuzzy prediction and feature selection of regression models.
Inspired by the idea of elastic net, Mokhtia et al. \cite{Mokhtia2021} 
studied a series of penalty models based on fuzzy correlation, 
these models can output sparse coefficient vectors 
through dual regularization to complete feature selection.
Based on the above work, we compromised $\ell_{2,1}$ regularization 
and $\ell_{F}$ regularization. That is,
when a certain number of features are introduced into the model, 
the model begins to consider the synergy of multiple features.

Based on the above questions, 
the main contributions of this paper are as follows:
\begin{itemize}
    \item Under the background of multi-label learning, 
    a joint sparse regularization term is proposed, 
    which can obtain a highly sparse coefficient matrix 
    and retain only the features with low correlation.
    
    \item A joint similarity matrix is constructed 
    by combining the similarity information of features and labels, 
    which effectively alleviates the "short circuit" phenomenon.
    
    \item Using random walk strategy, 
    the neighborhood graph with high sparsity 
    and robustness can be generated adaptively.
\end{itemize}

\section{Methods\label{sec:3}}

Let $ X = (x_{1}, ..., x_{n})^{T} \in \mathbb{R}^{n \times p} $ 
be the feature matrix and 
$ Y = (y_{1},...,y_{n}) ^{T} \in \mathbb{R}^{n \times m} $ 
be the label matrix, 
where $ x_{i} = (x_{i1}, ..., x_{ip}) $, 
$ y_{i} = (y_{i1}, ..., y_{im}) $, $ (x_{i}, y_{i}) $ is called an sample.
$ T $ represents the transpose of a matrix or vector,  
$ {\| * \| }_{2} $ is the $ \ell_{2} $-norm of vector,
$ {\| * \| }_{F} $ and $ {\| * \|}_{2, 1} $ are the Frobenius norm 
($ \ell_{F}-norm $) and $ \ell_{2,1} $-norm of matrix respectively. 
Here, $ {\| A \| }_{F} = tr( AA^{T}) $, 
$ {\| A \| }_{2, 1} = \sum \limits_{i} {\| A_{i} \|}_{2} $,
$ A_{i} $ is the $ i $-th row of matrix $ A $.

The prediction value of labeled samples should be larger than 
that of unlabeled samples, 
and the smaller the prediction error, 
the stronger the fitting ability of the model. 
Therefore, we take the sum of squares of errors as the basic loss function.

\begin{equation} \label{formula1}
    \min \limits_{W,b} \frac{1}{2}{\| X W + 1_{n} b - Y \|}_{F}^{2}
\end{equation}

where $ W \in \mathbb{R}^{p \times m} $ is the coefficient matrix, 
$ b \in \mathbb {R}^{1 \times m} $ is the bias vector, 
and $ 1_{n} $ is the column vector with all 1 element.

Samples and samples in the low-dimensional space should 
maintain the nature of adjacency, 
if they have adjacency in the original space. 
That is, 
${\Vert x_{i}-x_{j} \Vert}_{2} \propto {\Vert x_{i} W- x_{j} W \Vert}_{2}$.
In order to preserve this local structure, 
we add a manifold regularization term to the model.

\begin{equation} \label{formula2}
    \begin{split}
    LR(W) & =\frac{1}{2} \sum \limits_{i,j}{\| x_{i} W - x_{j} W\|}_{2}^{2} S_{ij} 
    = \frac{1}{2} \sum \limits_{i,j} ( x_{i} W - x_{j} W ) ( x_{i} W - x_{j} W )^{T}  \\
    & = \sum \limits_{i} x_{i} W ( x_{i} W )^{T} P_{ii} - \sum \limits_{i,j} x_{i} W ( x_{j} W )^{T} S_{ij} 
    = tr( W^{T} X^{T} L X W )
    \end{split} 
\end{equation}
where $ L = P - S $, $ P $ is a diagonal matrix, 
and $ P_{ii} = \sum \limits_{j} S_{ij} $.

\begin{figure}[htbp]\label{fig1}
    \centering
    \includegraphics[scale=0.25]{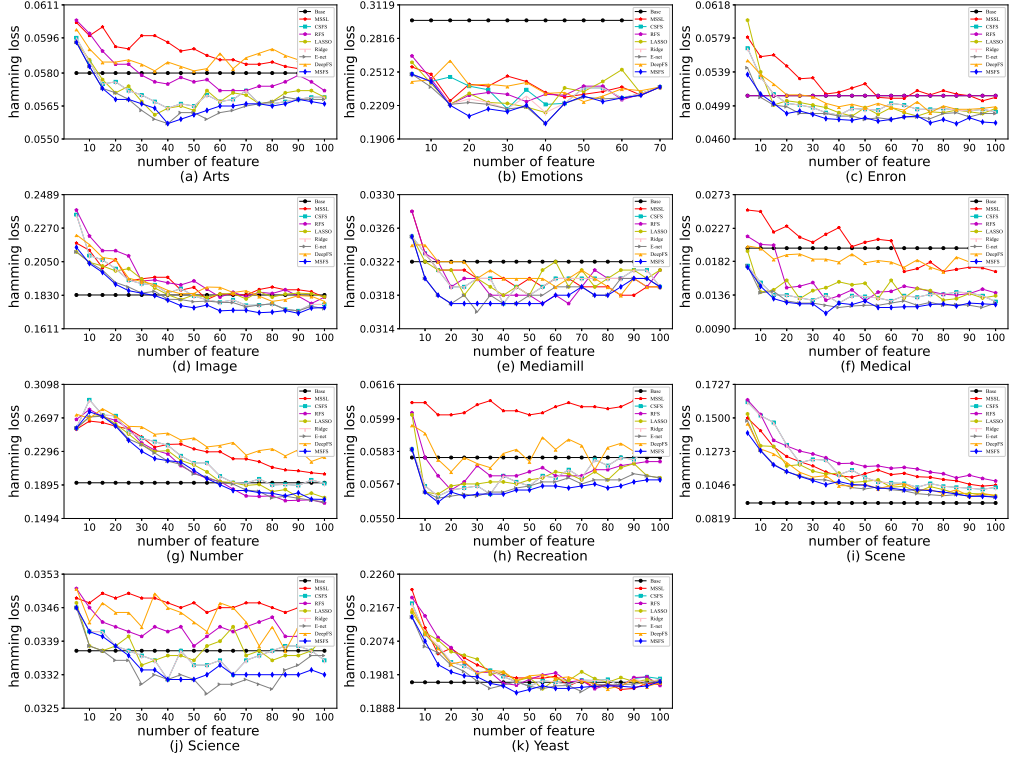}
    \caption{Hamming Loss comparisions of 8 feature selection algorithms on 11 datasets.}
\end{figure}

Usually, the neighborhood graph in the manifold regularization term is 
obtained directly through 
KNN \cite{Cai2010-MCFS,Cai2018-MSSL,He2004-LPP,Roweis2000-LLE,Zhang2019-MDFS}, 
so it is difficult to avoid the occurrence of "short circuit", 
especially on the data set like "Swiss Roll".
In supervised learning, especially in multi-label learning, 
the cost of obtaining labels is very expensive, 
but methods such as MSSL and MDFS only use label information once 
when training models,
which leads to the waste of label information.
Different from the above method, 
we first construct a joint similarity matrix using 
the information of labels and features, 
and then implement multi-step random walk based on 
this matrix to obtain the neighborhood graph. 
Joint similarity matrix can avoid "short circuit" phenomenon 
to a great extent, 
and random walk strategy can further improve the robustness of 
neighborhood graph to outliers. 
In addition, we only need to select a small number of random walk steps 
to get a highly sparse neighborhood graph, 
which greatly improves the computational efficiency.

The specific steps to obtain the neighborhood graph are as follows:

\noindent
    \textbf{step1}: Calculate the Euclidean distance matrix betwwen instances based on 
the features of instance $ Dist=(d_{ij})$, 
where $ d_{ij}={\| x_{i}-x_{j} \|}_{2}, \quad \forall i,j=1,2,\dots,n $.

\noindent
    \textbf{step2}: Calculate the Gaussian adjacency weight matrix 
$ V = (v_{ij}) $, 
where $ v_{ij}  =e^{ - \frac{d_{ij}^{2}}{ \sigma^{2}}},\quad \forall i,j=1,2,...,n $.

\begin{figure}[htbp]\label{fig2}
    \centering
    \includegraphics[scale=0.25]{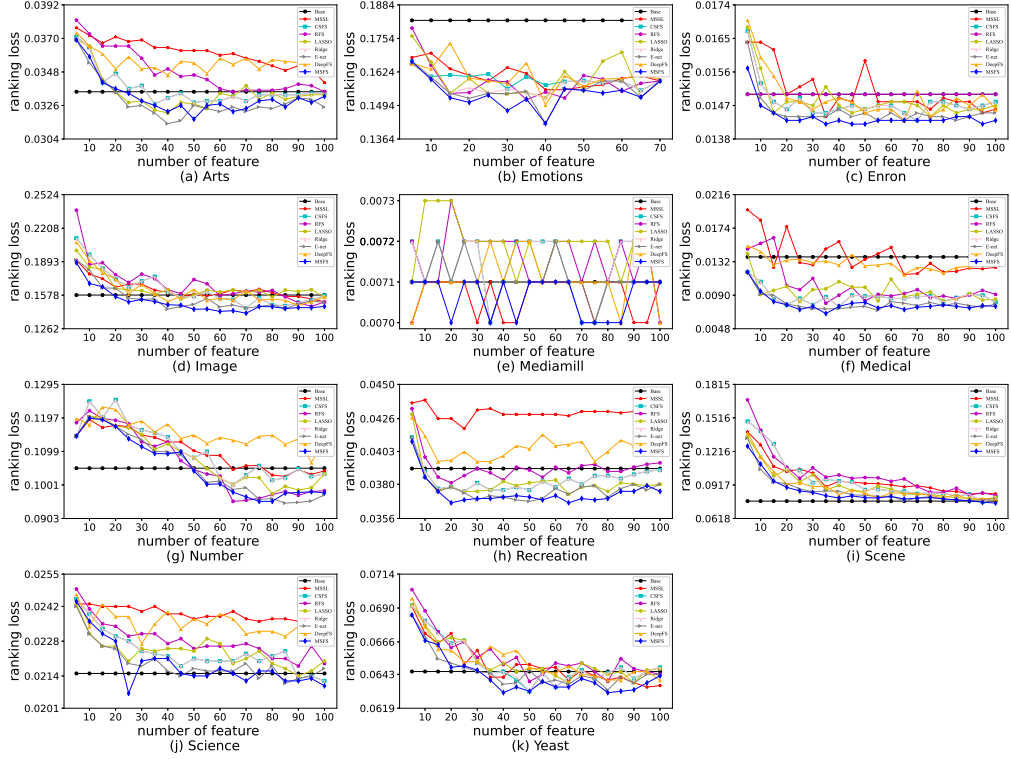}
    \caption{Ranking Loss comparisions of 8 feature selection algorithms on 11 datasets.}
\end{figure}

\noindent
    \textbf{step3}: Calculate Jaccard Index matrix $ R = (r_{ij}) $ between 
    samples based on labels. 

\begin{equation}\label{formula3}
    r_{ij}=
    \begin{cases}
        \frac{y_i y_j^T}{y_i y_i^T + y_j y_j^T - y_i y_j^T}, & i \neq j \\
        0, & otherwise.
    \end{cases}
\end{equation}

\noindent
    \textbf{step4}: Calculate the joint similarity matrix $ T = V \odot R $, 
the symbol $ \odot $ represents the Hadamard product, 
that is $ T_{ij} = V_{ij} \cdot R_{ij},\quad \forall i,j = 1,2,\dots,n $.

\noindent
\textbf{step5}: Starting from node $ i $, 
implement one-step random walk according to probability $ P_{i} $, 
where $ P = D^{-1}T $, $ D $ is a diagonal matrix and 
satisfies $ D_{ii} = \sum \limits_{j} T_{ij},\quad \forall i=1,2,\dots,n $. 
If it reaches the node $ j $, 
then implement a one-step random walk with probability $ P_{j} $ and 
repeat $ k $ times like this.

The counting matrix $ C $ can be obtained by recording the times that 
it passes through other nodes when starting from node $ i $. 
We only need to make $ S = (C + C^{T}) / 2 $ to get a 
highly robust neighborhood graph based on the joint structure 
and random walk strategy. 
Algorithm description please refer to Algorithm 1.

\begin{table}[htbp]\label{algorithm1}
  
    \centering
    \begin{tabular}{l}
    \hline
    \textbf{Algorithm 1}: Neighborhood Graph Algorithm (DFS and BFS) \\ 
    \hline
    \textbf{Input}: feature matrix $X\in\mathbb{R}^{n\times p}$, 
    label matrix $Y\in \mathbb{R}^{n \times m}$. \\
    \textbf{Output}: Neighborhood Graph $ S $.   \\
    Calculate one-step transform probability matrix $ P $.  \\
    Initialize $ C $ as a $ n \times n $ zero matrix.  \\
    \textbf{For} $ i=1:n $  \\
    \qquad Set $ x_{0} = x_{i} $, $ P_{0} = P_{i} $  \\
    \qquad \textbf{If DFS}:  \\
    \qquad \qquad \textbf{For} $ s=1:k $  \\
    \qquad \qquad \qquad Implement one-step random walk that starts 
    from $ x_{0} $ with the possibility of $ P_{0} $.  \\
    \qquad \qquad \qquad If it reached node $ x_{j} $, 
    then $ c_{ij} = c_{ij} + 1 $.  \\
    \qquad \qquad \qquad Set $ p_{ji} = 0 $ and renormalized $ P_{j} $, 
    so that $ \sum \limits_{i} p_{ji} = 1 $.  \\
    \qquad \qquad \qquad Set $ x_{0} = x_{j}, P_{0} = P_{j} $  \\ 
    \qquad \textbf{If BFS}:  \\
    \qquad \qquad \textbf{For} $ s=1:k $  \\
    \qquad \qquad \qquad Implement one-step random walk that starts 
    from $ x_{0} $ \\
    \qquad \qquad \qquad with the possibility of $ P_{0} $.  \\
    \qquad \qquad \qquad If it reached node $ x_{j} $, 
    then $ c_{ij} = c_{ij} + 1 $.  \\
    Set $S = (C + C^{T}) / 2$  \\
    \hline
    \textbf{Notes}: DFS and BFS are abbreviations for Depth-First Search and Breadth-First Search, respectively.
    \end{tabular}
\end{table}

\begin{figure}[htbp]\label{fig3} 
    \centering
    \includegraphics[scale=0.25]{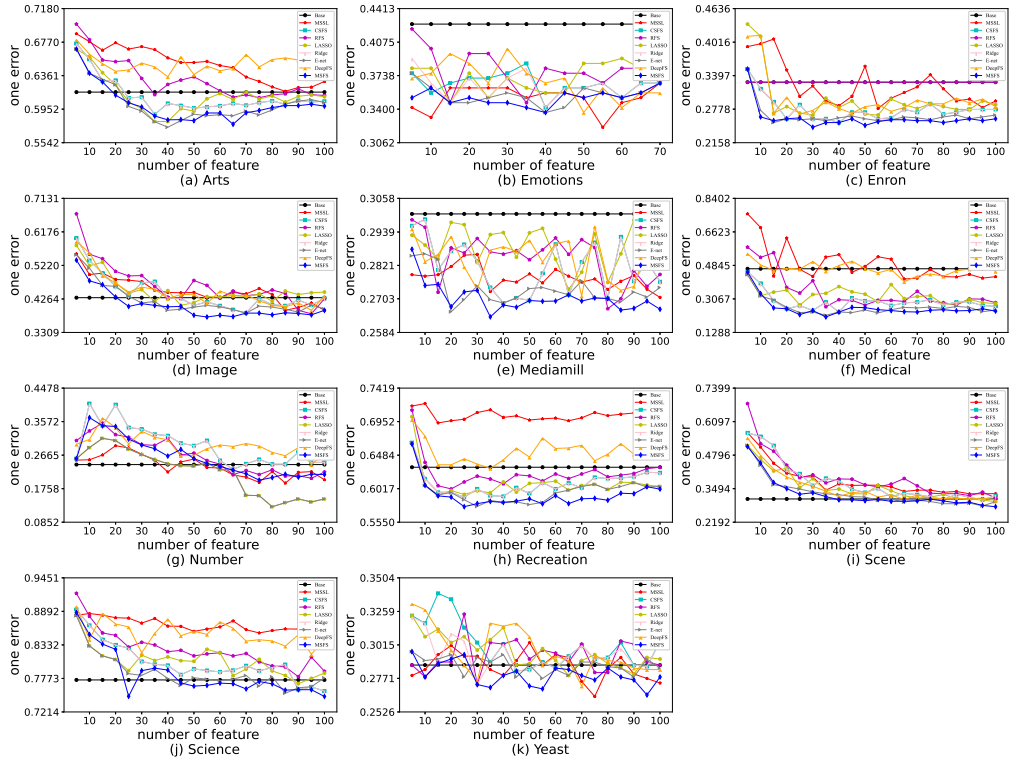}
    \caption{One Error comparisions of 8 feature selection algorithms on 11 datasets.}
\end{figure}

The neighborhood graph obtained by algorithm 1 not only has symmetry 
and sparsity, but also has strong robustness to outliers. 
In fact, carrying out random walks is equivalent to carrying out 
weighted random sampling, 
in which the weight of abnormal samples is 0 or very small.

By applying $\ell_{2,1}$ regularization to the coefficient matrix, 
it will become row sparse, thus realizing feature selection.

\begin{equation} \label{fomula4}
    SR_{1}(W) = {\| W \|}_{2,1}
\end{equation}

On the other hand, 
multicollinearity of features usually exists in high dimensional data sets. 
However, this effect can be reduced by imposing a quadratic penalty term, 
so we consider using Frobenius regularization term in the model 
at the same time.

\begin{equation} \label{formula5}
    SR_{2}(W) = { \| W \|}_{F}^{2}
\end{equation}

\begin{figure}[htbp]\label{fig4} 
    \centering
    \includegraphics[scale=0.25]{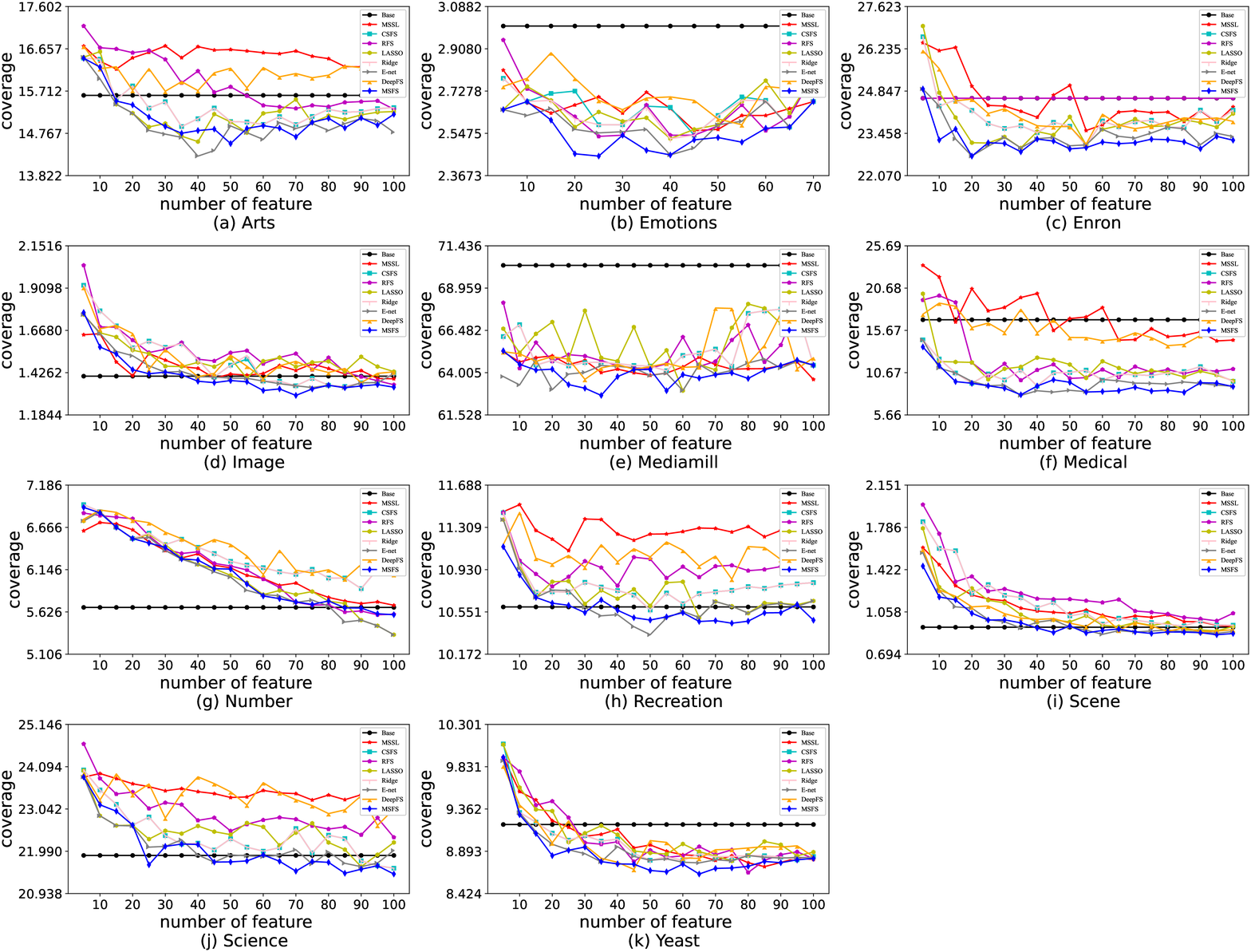}
    \caption{Coverage comparisions of 8 feature selection algorithms on 11 datasets.}
\end{figure}

Inspired by the elastic net method, we construct a joint regularization term 
by introducing the factor $\rho$, 
which can balance the effects of $SR_{1}(W)$ and $SR_{2}(W)$.
The model considers not only the action of a single feature, 
but also the joint action of a group of features.

\begin{equation} \label{formula6}
    \begin{split}
    SR(W) & = \rho SR_{1}(W) + (1-\rho)SR_{2}(W) \\
    & = \rho {\| W \|}_{2,1} + (1 - \rho){\| W \|}_{F}^{2}
    \end{split}
\end{equation}

\noindent
where $ {\| W \|}_{2,1} $ is not continuous and differentiable, 
so we need to find an approximate solution. Due to

\begin{equation} \label{formula7}
    {\| W \|}_{2,1} = \sum \limits_{i=1}^{p} \sqrt{\sum \limits_{j=1}^{m} W_{ij}^{2}} = \sum \limits_{i=1}^{p} \Vert w_{i} \Vert_{2} = \sum \limits_{i=1}^{p}(w_{i} w_{j}^{T})^{\frac{1}{2}}
\end{equation}

and when $ \forall i = 1, 2, \dots, p $ satisfies $ w_{i} \neq 0 $. Therefore, the following formula always holds

\begin{equation} \label{formula8}
    \begin{split}
      \frac{\partial {\| W \|}_{2,1}}{ \partial W} 
      & = \frac{\partial {\sum \limits_{i=1}^{p} (w_{i} w_{j}^{T})^{\frac{1}{2}}}}{\partial {w_{j}}} = 2UW \\
      & = \frac{\partial {tr(W^{T} U W)}}{\partial {W}}
    \end{split}
\end{equation}

\noindent
That is $\Vert W \Vert_{2,1} - tr(W^{T} U W)$ is an constant that can be 
omited in object function, 
so approximately there can be

\begin{equation}\label{formula9}
  {\| W \|}_{2,1} = tr(W^{T} UW)
\end{equation}

Here $ U $ is a diagonal matrix, 
and its $ i $-th element is $ U_{ii} = 1 / max \{ 2{\| w_{i} \|}_{2}, \varepsilon\} $, 
$ \varepsilon $ is a sufficiently small positive constant 
such that the denominator is never equal to 0.

\begin{figure}[htbp]\label{fig5}
    \centering
    \includegraphics[scale=0.25]{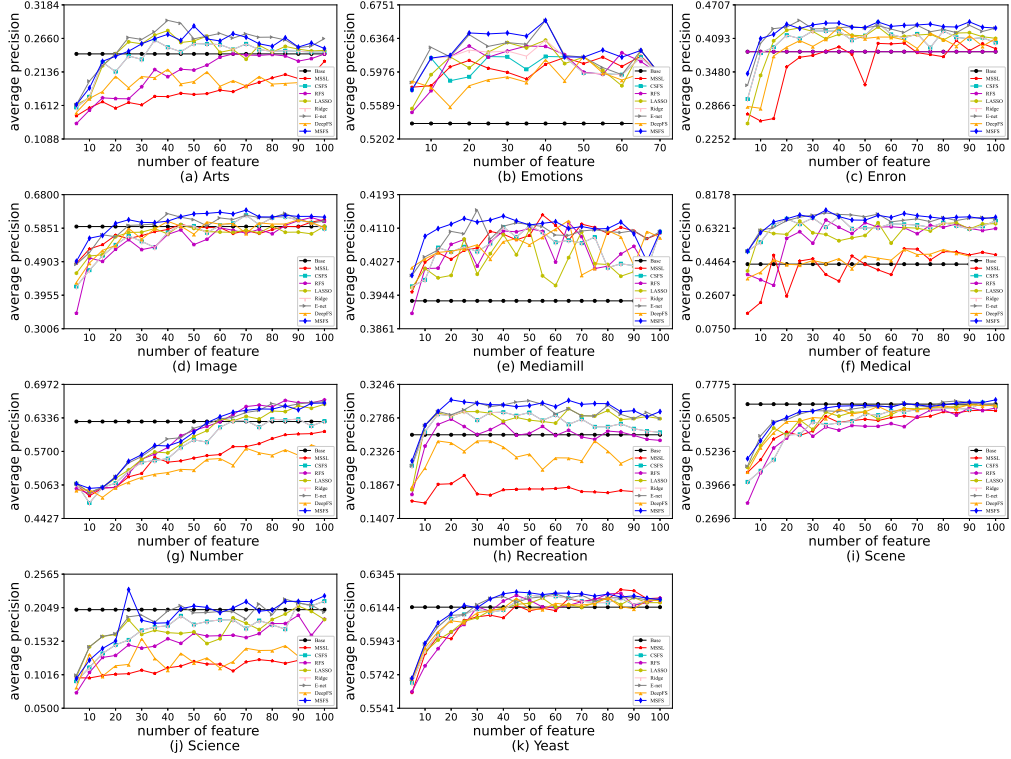}
    \caption{Average Precision comparisions of 8 feature selection algorithms on 11 datasets.}
\end{figure}

Based on the above, we get the following optimization problem
\begin{equation} \label{formula10}
    \begin{split}
    \min \limits_{W,b} 
    & \frac{1}{2}{\| X W + 1_{n} b - Y \|}_{F}^{2} + \frac{\alpha}{2} \cdot \frac{1}{2} \sum \limits_{i,j}{\| x_{i} W - x_{j} W \|}_{2}^{2} S_{ij} + \frac{\beta}{2} ( \rho {\| W \|}_{2,1} + ( 1 - \rho){\| W \|}_{F}^{2})
    \end{split}
\end{equation}

Using the trace of the matrix to reformulate the above formula as follows
\begin{equation} \label{formula11}
    \begin{split}
    \min \limits_{W,b}
    & \frac{1}{2} tr(W^{T} X^{T} X W) + tr(b^{T} 1_{n}^{T} X W) - tr(Y^{T} X W) + \frac{1}{2} tr(b^{T} 1_{n}^{T} 1_{n} b) - tr(Y^{T} 1_{n} b) \\
    + & \frac{1}{2}tr(Y^{T} Y) + \frac{\alpha}{2}tr(W^{T} X^{T} L X W) +  \frac{\beta \rho}{2}tr(W^{T} U W) + \frac{\beta \left( 1 - \rho \right)}{2}tr(W^{T}W)
    \end{split}
\end{equation}

\section{Optimization algorithm and analysis\label{sec:4}}
\subsection{Optimization algorithm}

May wish to write the formula (9) as $ f(W,b) $, 
then we can solve b and W by using alternating least squares(\textbf{ALS}) methods.

\noindent
\textbf{(1) Fixed W and U, Update b}
\begin{equation} \label{formula12}
    b = \frac{1}{n} ( 1_{n}^{T}Y - 1_{n}^{T} X W )
\end{equation}

\noindent
\textbf{(2) Fixed U and b, Update W}
\begin{equation} \label{formula13}
    W = {(X^{T} (H + \alpha L) X + \beta ( 1 - \rho) I_{p} + \beta \rho U)}^{-1} X^{T} H Y
\end{equation}
where $ H = I_{n} - \frac{1}{n} 1_{n} 1_{n}^{T} $.

\noindent
\textbf{(3) Fixed W and b, Update U}
\begin{equation} \label{formula14}
  U_{i,i} = \frac{1}{max\{ 2 \| w_{i} \|_{2}, \varepsilon \}}
\end{equation}

The description of the optimization algorithm corresponding to 
the above process please refer to the Algorithm2.

\begin{table}[htbp] \label{algorithm2}
    \centering
    \begin{tabular}{l}
    \hline
    \textbf{Algorithm 2}: Random \textbf{M}anifold Sampling and Joint \textbf{S}parse Regularization 
    for Multi-label \textbf{F}eature \textbf{S}election  \\ 
    \hline
    \textbf{Input:} feature matrix: $ X \in \mathbb{R}^{n \times p} $, 
    label matrix: $ Y \in \mathbb {R}^{n \times m} $, \# of select feature: $ l $, parameter: $ \alpha, \beta, \rho $. \\
    \textbf{Output}: the set of select features: $ SF $.  \\
    Calculate the neighborhood graph S according to \textbf{Algorithm 1}.  \\
    Calculate the graph Laplacian matrix of the sample graph $ L $.  \\
    Calculate the centering matrix $ H = I_n - 1_n 1_n^T / n$.  \\
    Set $ t = 0, \quad  \varepsilon = 1 \times 10^{-64} $.  \\
    Random initializing $ U $ as a diagonal matrix.  \\
    \textbf{While} not convergence  \\
    \qquad Update $ W $ as $ W^{(t+1)} $ according to the formula (\ref{formula13}).  \\
    \qquad Update $ U $ as $ U^{(t+1)} $ according to the formula (\ref{formula14}).  \\
    \qquad $ t = t + 1 $.  \\
    \textbf{Until} convergence  \\
    Calculate features weight vector $ scores $, 
    where $ scores_{i} = { \| w_i \|}_2 $.  \\
    Select the top $ l $ features with the highest score.  \\ 
    \hline
    \end{tabular}
\end{table}

\begin{figure}[htbp]\label{fig6} 
  \centering
  \includegraphics[scale=0.25]{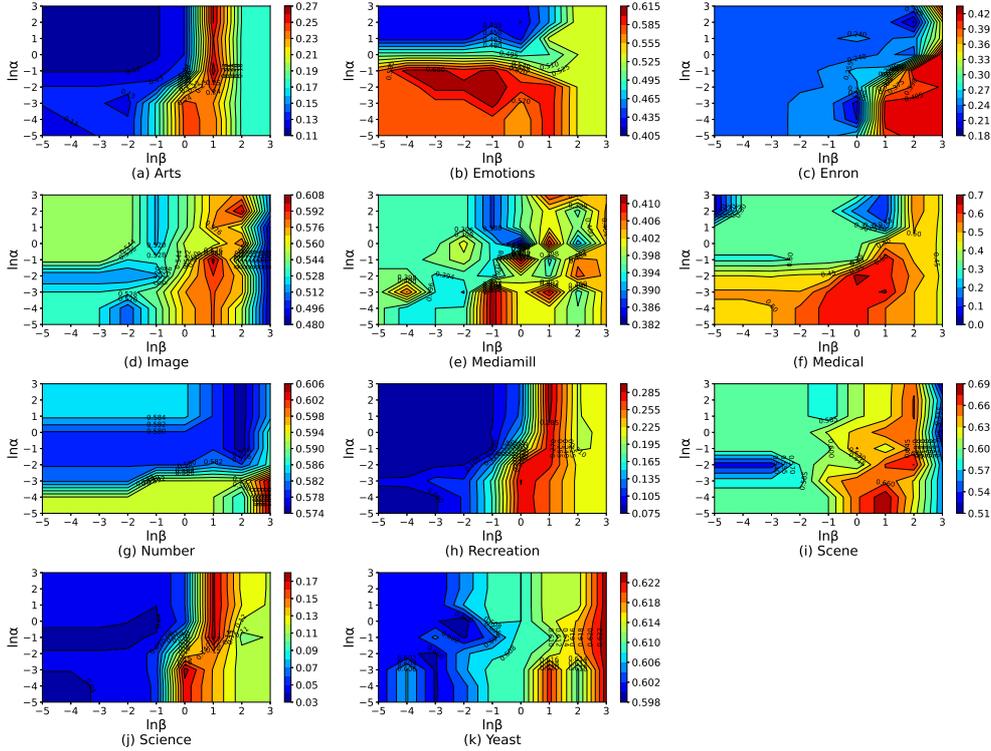}
  \caption{Average Precision of MSFS when $\alpha$ and $\beta$ varies from $10^{-5}$ to $10^{3}$, the number of feature set as 50.}
\end{figure}

\subsection{Complexity analysis}

We denote $n, p$ and $ m $ as the number of instances, 
features and labels respectively, 
$ k $ and $ t $ are the number of random walk steps and iterations.
In general, $ n > p > m $, $ k = 80 $, and $ t \leq 50 $  are satisfied. 
Therefore, the complexity of algorithm 1 is $ O (( p + m ) n^{2} + ( p + m + k) n) $, 
and the computational complexity of each iteration of algorithm 2 is 
$ O(p n^{2} + (p^{2} + m p) n + p^{3} + p^{2} + m p) $.
Therefore, the total complexity of our algorithm is 
$ O((p + m) n^{2} + (p + m + k) n) + O(t ( p n^{2} + (p^{2} + m p) n + p^{3} + p^{2} + m p)) = O(n^{2}p + np^{2} + p^{3}) $.

\begin{figure}[htbp]\label{fig7} 
  \centering
  \includegraphics[scale=0.25]{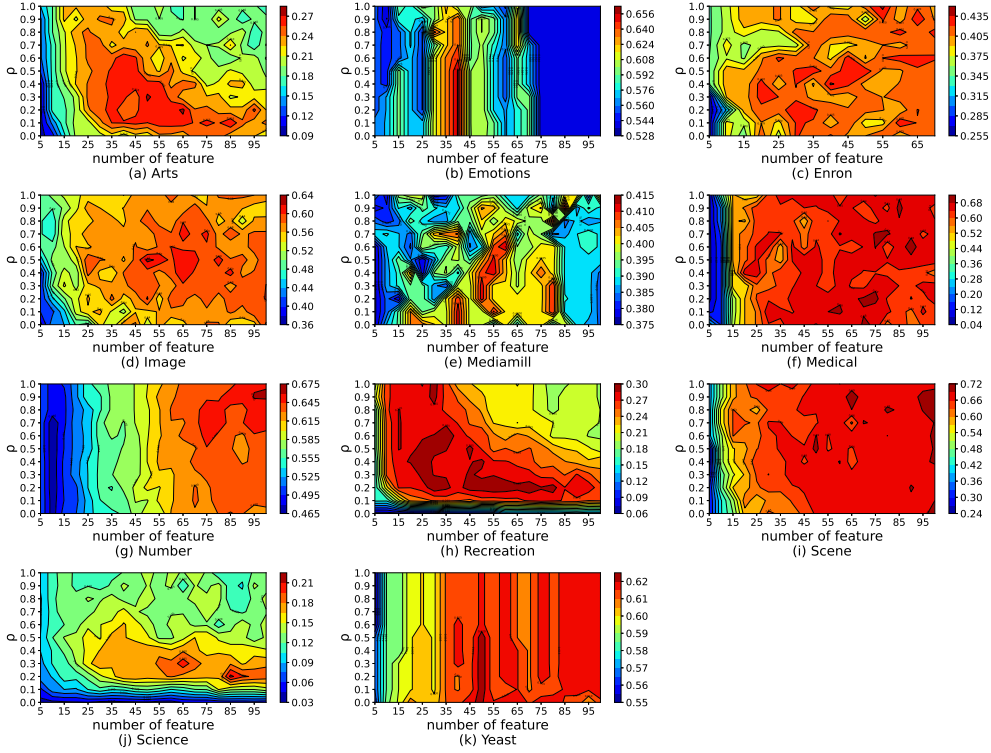}
  \caption{Average Precision of MSFS when $\rho$ varies from 0.0 to 1.0 and the number of feature varies from 5 to 100 (70 for emotions).}
\end{figure}

\subsection{Convergence analysis}

In this section, we will prove the convergence of Algorithm 2. 
Before the proof, we need to introduce a lemma.

\begin{lemma} 
  For any $ a > 0, b > 0 $, the following inequality always holds
  \begin{equation} \label{formula15}
    \sqrt{a} - \frac{a}{2 \sqrt{b}} \leq \sqrt{b} - \frac{b}{2 \sqrt{b}}
  \end{equation}
\end{lemma}

\begin{theorem}
  Algorithm 2 gives the optimal solution with convexity.
\end{theorem} 

\begin{proof}[\textbf{\upshape Proof:}]
    
  If we write $ W^{(t+1)}, b^{(t+1)} $ as the solution produced in the $ t $-th iteration, we can get the following relationship
  \begin{equation} \label{formula16}
      \begin{split}
          \left( W^{(t+1)}, b^{(t+1)} \right) 
          & = \arg \min \limits_{W,b} \frac{1}{2} {\left \| XW + 1_{n}b -Y \right \|}_{F}^{2} + \frac{\alpha}{2} tr \left(W^{T} X^{T} L X W \right) + \frac{\beta \left( 1 - \rho \right)}{2} tr \left( W^{T} W \right) + \frac{\beta \rho}{2} tr \left( W^{T} U^{(t)} W \right)
      \end{split}
  \end{equation}

  Here $ U^{(t)} $ is a diagonal matrix, and its $ i $-th diagonal element is 
  $ \frac{1}{2 {\left\| w_{i}^{(t)} \right \|}_{2}} $, so the following formula can be obtained.
  \begin{equation} \label{formula17}
      \begin{split}
          & \frac{1}{2} {\left\| X W^{(t+1)} - 1_{n} b^{(t+1)} - Y  \right\|}_F^{2} + \frac{\alpha}{2} tr \left( {\left( W^{(t+1)} \right)}^{T} X^{T} L X W^{(t+1)} \right) 
          + \frac{\beta \left( 1 - \rho \right) }{2} tr \left( {\left( W^{(t+1)} \right) }^{T} W \right) \\
          & + \frac{\beta \rho}{2} tr \left( {\left( W^{(t+1)} \right)}^{T} U^{t} W^{(t+1)}  \right) \leq \frac{1}{2} {\left\| X W^{(t)} - 1_{n}b^{(t)} - Y  \right\| }_F^{2} + \frac{\alpha}{2} tr \left( {\left( W^{(t)} \right) }^{T} X^{T} L X W^{(t)} \right) \\
          & + \frac{\beta \left( 1 - \rho \right) }{2} tr \left( {\left( W^{(t)} \right) }^{T}W \right) + \frac{\beta \rho}{2} tr \left( {\left( W^{(t)} \right) }^{T} U^{t} W^{(t)}  \right) 
      \end{split}
  \end{equation}

  Further, there is
  \begin{equation} \label{formula18}
      \begin{split}
        & \frac{1}{2} {\left\| X W^{(t+1)} - 1_{n} b^{(t+1)} - Y  \right\| }_F^{2} + \frac{\alpha}{2} tr \left( {\left( W^{(t+1)} \right) }^{T} X^{T} L X W^{(t+1)} \right) + \frac{\beta \left( 1 - \rho \right) }{2} tr \left( {\left( W^{(t+1)} \right) }^{T} W \right) + \frac{\beta \rho}{2} \sum \limits_{i} \frac{{\left\| w_{i}^{(t+1)} \right\| }_{2}^{2}}{2{\left\| w_{i}^{(t)} \right\| }_{2}}\\
        & \leq \frac{1}{2} {\left\| X W^{(t)} - 1_{n} b^{(t)} - Y  \right\| }_F^{2} + \frac{\alpha}{2} tr \left( {\left( W^{(t)} \right) }^{T} X^{T} L X W^{(t)} \right) + \frac{\beta \left( 1 - \rho \right) }{2} tr \left( {\left( W^{(t)} \right)}^{T} W \right) + \frac{\beta \rho}{2} \sum \limits_{i} \frac{{\left\| w_{i}^{(t)} \right\| }_{2}^{2}}{2{\left\| w_{i}^{(t)} \right\| }_{2}}
      \end{split}
  \end{equation}

  Therefore, it can be inferred that the following formula holds
  \begin{equation} \label{formula19}
      \begin{split}
        & \frac{1}{2} {\left\| X W^{(t+1)} - 1_{n} b^{(t+1)} - Y  \right\| }_F^{2} + \frac{\alpha}{2} tr \left( {\left( W^{(t+1)} \right) }^{T} X^{T} L X W^{(t+1)} \right) + \frac{\beta \left( 1 - \rho \right)}{2} tr\left( {\left( W^{(t+1)} \right)}^{T} W \right) + \frac{\beta \rho}{2}{\left\| W^{(t+1)} \right\| }_{2,1} \\
        & - \frac{\beta \rho}{2} \left( {\left\| W^{(t+1)} \right\| }_{2,1} - \sum \limits_{i} \frac{{\left\| w_{i}^{(t+1)} \right\| }_{2}^{2}}{2{\left\| w_{i}^{(t)} \right\| }_{2}} \right) \leq \frac{1}{2} {\left\| X W^{(t)} - 1_{n} b^{(t)} - Y  \right\| }_F^{2} + \frac{\alpha}{2} tr \left( {\left( W^{(t)} \right)}^{T} X^{T} L X W^{(t)} \right) \\
        & + \frac{\beta \left( 1 - \rho \right) }{2} tr \left( {\left( W^{(t)} \right) }^{T} W \right) + \frac{\beta \rho}{2}{\left\| W^{(t)} \right\| }_{2,1} - \frac{\beta \rho}{2} \left( {\left\| W^{(t)} \right\| }_{2,1} - \sum \limits_{i} \frac{{\left\| w_{i}^{(t)} \right\| }_{2}^{2}}{2{\left\| w_{i}^{(t)} \right\| }_{2}} \right)
      \end{split}
  \end{equation}

  From $ { \| W \| }_{2,1} = \sum \limits_{i} { \| w_{i} \| }_{2}$ and the \textbf{Lemma 1}
  \begin{equation} \label{formula20}
      { \| w_{i}^{(t+1)} \| }_{2} - \frac{{\| w_{i}^{(t+1)} \| }_2^{2}}{2{ \| w_i^{(t)} \|}_{2}} 
      \leq 
      { \| w_{i}^{(t)} \|}_{2} - \frac{{ \| w_{i}^{(t)} \|}_2^{2}}{2{ \| w_i^{(t)} \|}_{2}}
  \end{equation}

  Thus
  \begin{equation} \label{formula21}
      \sum \limits_{i} {\left( { \| w_{i}^{(t+1)} \| }_{2} - \frac{{ \| w_{i}^{(t+1)} \| }_2^{2}}{2{ \| w_i^{(t)} \| }_{2}} \right)} 
      \leq 
      \sum \limits_{i} {\left({\| w_{i}^{(t)} \| }_{2} - \frac{{\| w_{i}^{(t)} t \| }_2^{2}}{2{\| w_i^{(t)} \| }_{2}} \right) }
  \end{equation}

  That is
  \begin{equation} \label{formula22}
      {\| W^{(t+1)} \| }_{2,1} - \sum \limits_{i}\frac{{\| w_{i}^{(t+1)} \|}_2^{2}}{2{\| w_i^{(t)} \| }_{2}}
      \leq
      {\| W^{(t)} \| }_{2,1} - \sum \limits_{i}\frac{{\| w_{i}^{(t)} \| }_2^{2}}{2{\| w_i^{(t)} \| }_{2}}
  \end{equation}

  Therefore we can immediately infer that the algorithm converges.
\end{proof}

\begin{figure}[htbp]\label{fig8} 
    \centering
    \includegraphics[scale=0.25]{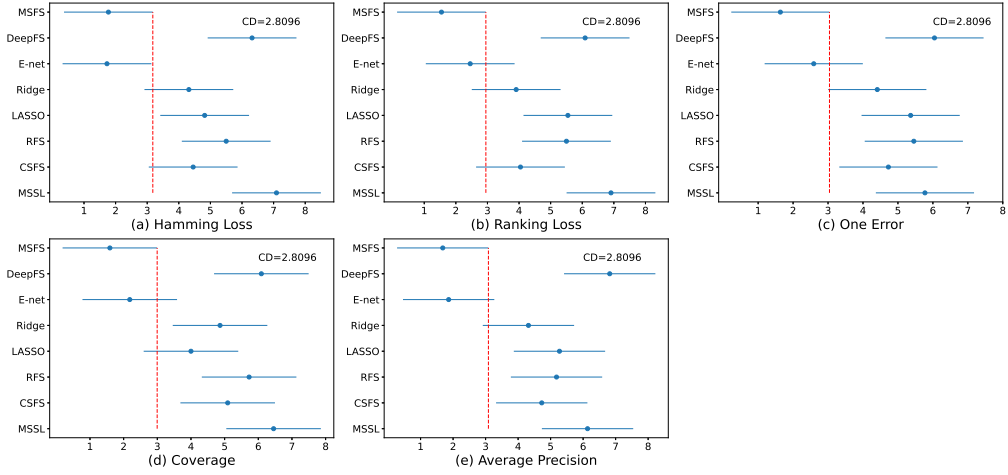}
    \caption{Comparison of MSFS against other comparing methods with the Bonferroni-Dunn test($\alpha=0.05$).}
\end{figure}

\section{Experiments and results\label{sec5}}

In this section, a series of compared experiments are designed 
in six publicly available data sets.
Experimental results show that the proposed method achieves 
state-of-art on multiple data sets.
It should be noted that in the experiments involved in this paper, 
ML-KNN is used as the basic classifier.

\subsection{Datasets}

We experimented with 11 datasets for descriptions in Table 1, 
all of which are freely available from 
the open source multi-label learning data site but the Number.
The data set of Number is made manually, and it contains 2000 pieces of data. 
Each piece of data consists of five digital pictures containing 81 pixels, 
and some of the five pictures may not contain numbers. 
In order to make the classifier trained from the data set more robust, 
we add 15\% Gaussian noise to each data set.

\begin{table}[htbp]
    \centering
    \caption{Data sets and the related descriptions} \label{table1}
    \resizebox{\textwidth}{!}{
        \begin{tabular}{lllllllll}
            \hline
            Datasets  & Dim(D)  & L(D)  & $\|$D$\|$  & $\|$train$\|$  & $\|$test$\|$  & PMC      & ANL      & Dens         \\ 
            \hline
            Arts      & 462     & 26    & 5000       & 2000           & 3000          & 0.4398   & 1.6360   & 0.0629       \\
            Emotions  & 72      & 6     & 593        & 391            & 202           & 0.6998   & 1.8685   & 0.3114       \\
            Enron     & 1001    & 53    & 1702       & 1123           & 579           & 0.8848   & 3.3784   & 0.0637       \\
            Image     & 294     & 5     & 2000       & 1000           & 1000          & 0.2285   & 1.2360   & 0.2472       \\
            Mediamil  & 120     & 101   & 15730      & 5000           & 10730         & 0.9350   & 4.3756   & 0.0433       \\
            Medical   & 1449    & 45    & 978        & 333            & 645           & 0.2311   & 1.2454   & 0.0277       \\
            Number    & 405     & 10    & 2000       & 1000           & 1000          & 0.9825   & 3.3265   & 0.3326       \\
            Recreation& 606     & 22    & 5000       & 2000           & 3000          & 0.3080   & 1.4232   & 0.0647       \\
            Scene     & 294     & 6     & 2407       & 1211           & 1196          & 0.0735   & 1.0740   & 0.1790       \\
            Science   & 743     & 40    & 5000       & 2000           & 3000          & 0.3228   & 1.4506   & 0.0363       \\
            Yeast     & 103     & 14    & 2417       & 1500           & 917           & 0.9868   & 4.2371   & 0.3026       \\ 
            \hline
            \multicolumn{9}{l}{(\url{http://mulan.sourceforge.net/datasets-mlc.html})}
        \end{tabular}
    }
\end{table}

Here, Dim(D) is the dimension of the corresponding data set, that is, 
the number of features. 
L(D) is the number of labels, 
$\|$D$\|$ is the number of instances of the data set. 
PMC, ANL and Dens are defined as follows:

\noindent
(1) $ PMC = \frac {1}{n} \sum \limits_{i=1}^{n} \{ y_{i} y_{i}^{T} \geq 2 \} $, which measures the percentage of documents belonging to more than one category.

\noindent
(2) $ ANL = \frac {1}{n} \sum \limits_{i=1}^{n} \vert y_{i} \vert$, which denotes the average number of labels that each object belongs to.

\noindent
(3) $ Dens = \frac{1}{n m} \sum \limits_{i=1}^{n} \vert y_{i} \vert$, which measures the density of the label distribution of the dataset.

\subsection{Evaluation metrics}

In multi-label learning, 
we pay more attention to the ranking of prediction results, 
so we use hamming loss, ranking loss, one error, coverage and 
average precision as the evaluation metrics of multi-label feature selection.

\noindent
(1) Hamming Loss measures the ratio of incorrectly predicted labels.
\begin{equation} \label{formula23}
    HL = \frac {1}{n} \sum \limits_{i=1}^{n} \frac {\hat{y}_{i} \oplus y_{i}}{m}
\end{equation}

\noindent
(2) Ranking Loss evaluates the average fraction of label pairs that irrelevant labels are ranked higher than the relevant label.
\begin{equation} \label{formula24}
    RL = \frac {1}{n} \sum \limits_{i=1}^{n} \frac{| \left\{ \left( l_{k},l_{j} \right) | f_{k} \left( x_{i} \right) < f_{j} \left( x_{i} \right), \left( l_{k}, l_{j} \in y_{i} \times \bar{y_{i}} \right) \right\} |}{\left| y_{i} \right| \left| \bar{y_{i}} \right|}
\end{equation}

\noindent
(3) One Error evaluates how many times the top-ranked label is not in the relevant label set of the instance.
\begin{equation} \label{formula25}
    OE = \frac{1}{n} \sum \limits_{i=1}^{n} \left( arg \max \limits_{l_{k} \in y_{i}} rank \left( x_{i}, l_{k} \right) \notin y_{k} \right)
\end{equation}

\noindent
(4) Coverage is used to measure the steps, on average, required to cover the true label.
\begin{equation} \label{formula26}
    Cov = \frac {1}{m} \left( \frac {1}{n} \sum \limits_{i=1}^{n} arg{\max \limits_{l_{k} \in y_{i}}} rank \left( x_{i}, l_{k} \right) - 1 \right)
\end{equation}

\noindent
(5) Average Precision evaluates the average fraction of labels ranked above a particular label in the truth label sets.
\begin{equation} \label{formula27}
    AP = \frac{1}{n} \sum \limits_{i=1}^{n} \frac {1}{y_{i}} \sum \limits_{l_{k} \in y_{i}} \frac{\left\{ l_{j} | rank \left( x_{i},l_{j} \right) \leq rank \left( x_{i}, l_{k} \right), l_{j} \in y_{i} \right\} }{rank \left( x_{i}, l_{k} \right) }
\end{equation}

\noindent
(6) The difference between two methods is distinguished with the critical difference (CD), as follow

\begin{equation} \label{formula28}
    CD = q_{\alpha} \sqrt{\frac{k(k+1)}{6N}}
\end{equation}
where $q_{\alpha}=2.690$ at significance level $\alpha= 0.05$, and then we can calculate CD = 2.8096 (k = 8, N = 11).

When using hamming loss, ranking loss, one error, coverage to evaluate, the smaller the value, the better the model performance, whereas on AP the opposite is true.

\subsection{Experiment setting}
    In order to verify the effectiveness of the proposed method, we make a comprehensive comparison with the following methods.

\noindent
\textbf{Base}: All original features are selected for learning tasks, the base classifier is ML-KNN \cite{Zhang2007-ML-KNN}.

\noindent
\textbf{RFS} \cite{Nie2010-RFS}: By imposing $\ell_{2,1}$ penalty on the loss function, 
the coefficient matrix with row sparse property is obtained, 
thus completing feature selection in multi-label learning.

\noindent
\textbf{MSSL} \cite{Cai2018-MSSL}: The method combines multi-label learning and manifold learning. 
It treats each feature as a node and constructs the graph Laplacian matrix. 

\noindent
\textbf{CSFS} \cite{chang2014-CSFS}: CSFS extends the algorithm to semi-supervised learning by weighting samples, 
which makes the model can be applied to large-scale data sets

\noindent
\textbf{LASSO} \cite{Tibshirani1996-LASSO}: It use the $ \ell_{1} $ regularization to make elements of the coefficient matrix become 0 as much as possible. 

\noindent
\textbf{Ridge} \cite{Hoerl1970-Ridge}: By imposing $\ell_{2}$ penalty on the coefficients, Ridge method makes the weights of variables with linear correlation decrease at the same time, 
thus weakening the adverse consequences caused by multiple linearity.

\noindent
\textbf{E-net} \cite{Zou2005-E-net}: It makes the model have the ability to select a single strong feature and a group of strong features by balancing $\ell_{1}$ and $\ell_{2}$ regularization.

\noindent
\textbf{DeepFS} : Under the background of deep learning, it can learn a large number of parameters by building a multi-layer network to evaluate the importance of features to labels. (We completed the experiment by referring to python's dl-selection library) 

For the sake of fairness, 
all of the parameters of the above mentioned methods for 
comparative experiments are generally 
set as $\{10^{-5}, 10^{-4}, 10^{-3}, 10^{-2},\\ 
10^{-1}, 1, 10^{1}, 10^{2}, 10^{3}\}$, 
and the iteration times of the algorithm are consistently set as 50. 
Specifically, 
for MLKNN, we set k=7, for E-net and our proposed MSFS, 
we set the value interval of $\rho$ 
as $\{0.0, 0.1, 0.2, 0.3, 0.4, 0.5, 0.6, 0.7, 0.8, 0.9,\\
1.0\}$. In addition, the length of random walk is set as 80 for MSFS. 
In the DeepFS, 
the number of layers and dimensions of the hidden layers of 
the neural network are selected from [2, 4, 8, 16] 
and [512, 256, 128, 64, 32].

\subsection{Comparison and analysis of methods}

    We conducted comparative experiments on 11 data sets, 
and we selected \{5, 10, 15,...,90, 95, 100\} for the number of features. 
For each method, we obtained its best performance through grid search 
and recorded it in the following tables. 
Table 2, 3, 4, 5, and 6 are the recorded results of hamming loss, 
ranking loss, one error, coverage, and average precision, respectively.
In these table, we use bold font to indicate the best performance.

\begin{table}[htbp]
    \centering
    \caption{Hamming loss(mean$\pm$std) of 8 different algorithms on 11 different datasets} \label{table2}
    \resizebox{\textwidth}{!}{
        \begin{tabular}{llllllllll}
            \hline
            Datasets      & Base      & MSSL    & CSFS    & RFS     & LASSO   & Ridge   & E-net   & DeepFS   & MSFS    \\ 
            \hline
            Arts		      & 0.0580	    & 0.0582$\pm$0.0014       & 0.0564$\pm$0.0016     & 0.0572$\pm$0.0011     & 0.0561$\pm$0.0020     & 0.0564$\pm$0.0013     & 0.0557$\pm$0.0015     & 0.0581$\pm$0.0017     & \textbf{0.0557$\pm$0.0014}  \\
            Emotions		  & 0.2979		& 0.2252$\pm$0.0096       & 0.2178$\pm$0.0060     & 0.2219$\pm$0.0061     & 0.2145$\pm$0.0073     & 0.2145$\pm$0.0072     & 0.2046$\pm$0.0062     & 0.2244$\pm$0.0053     & \textbf{0.2046$\pm$0.0048}  \\
            Enron		      & 0.0511		& 0.0505$\pm$0.0013       & 0.0488$\pm$0.0010     & 0.0511$\pm$0.0012     & 0.0487$\pm$0.0013     & 0.0488$\pm$0.0015     & 0.0483$\pm$0.0012     & 0.0491$\pm$0.0012     & \textbf{0.0478$\pm$0.0009}  \\
            Image		      & 0.1832		& 0.1802$\pm$0.0072       & 0.1728$\pm$0.0050     & 0.1776$\pm$0.0038     & 0.1802$\pm$0.0057     & 0.1728$\pm$0.0057     & 0.1728$\pm$0.0049     & 0.1790$\pm$0.0051     & \textbf{0.1712$\pm$0.0031}  \\
            Mediamill		  & 0.0322		& 0.0318$\pm$0.0011       & 0.0317$\pm$0.0010     & 0.0317$\pm$0.0011     & 0.0318$\pm$0.0010     & 0.0317$\pm$0.0010      & \textbf{0.0316$\pm$0.0008}     & 0.0318$\pm$0.0010     & 0.0317$\pm$0.0004  \\
            Medical		      & 0.0200		& 0.0168$\pm$0.0006       & 0.0125$\pm$0.0003     & 0.0128$\pm$0.0003     & 0.0129$\pm$0.0004     & 0.0125$\pm$0.0005     & 0.0119$\pm$0.0003     & 0.0168$\pm$0.0005     & \textbf{0.0111$\pm$0.0003}  \\
            Number		      & 0.1921		& 0.2026$\pm$0.0078       & 0.1889$\pm$0.0054     & \textbf{0.1679$\pm$0.0052}     & 0.1743$\pm$0.0032     & 0.1889$\pm$0.0054     & 0.1693$\pm$0.0037     & 0.2176$\pm$0.0041     & 0.1722$\pm$0.0033  \\
            Recreation		  & 0.0580		& 0.0601$\pm$0.0020       & 0.0561$\pm$0.0014     & 0.0564$\pm$0.0015     & 0.0562$\pm$0.0016     & 0.0561$\pm$0.0016     & 0.056$\pm$0.0015      & 0.0573$\pm$0.0015     & \textbf{0.0558$\pm$0.0012}  \\
            Scene		      & \textbf{0.0924}		& 0.1038$\pm$0.0038       & 0.1016$\pm$0.0023     & 0.1073$\pm$0.0027     & 0.0973$\pm$0.0030      & 0.1016$\pm$0.0031     & 0.0962$\pm$0.0025     & 0.0971$\pm$0.0025     & 0.0964$\pm$0.0016  \\
            Science		      & 0.0337		& 0.0345$\pm$0.0007       & 0.0331$\pm$0.0008     & 0.0338$\pm$0.0008     & 0.0334$\pm$0.0009     & 0.0331$\pm$0.0013     & \textbf{0.0328$\pm$0.0009}     & 0.0337$\pm$0.0011     & 0.0331$\pm$0.0008  \\
            Yeast		      & 0.1960		& 0.194$\pm$0.0062        & 0.1949$\pm$0.0050     & 0.1942$\pm$0.0053     & 0.1942$\pm$0.0059     & 0.1949$\pm$0.0072     & 0.1934$\pm$0.0050     & 0.1942$\pm$0.0063     & \textbf{0.1931$\pm$0.0047}  \\
            \hline
            \multicolumn{10}{l}{The smaller the value, the better the performs}
        \end{tabular}
    }
\end{table}

\begin{table}[htbp]
    \centering
    \caption{Ranking loss(mean$\pm$std) of 8 different algorithms on 11 different datasets} \label{table3}
    \resizebox{\textwidth}{!}{
        \begin{tabular}{llllllllll}
            \hline
            Datasets    & Base    & MSSL    & CSFS    & RFS     & LASSO   & Ridge   & E-net    & DeepFS    & MSFS        \\ 
            \hline
            Arts		    & 0.0335		& 0.0341$\pm$0.0009       & 0.0327$\pm$0.0009     & 0.0335$\pm$0.0009     & 0.0321$\pm$0.0010     & 0.0327$\pm$0.0010     & \textbf{0.0314$\pm$0.0006}     & 0.0346$\pm$0.0010     & 0.0317$\pm$0.0010	\\
            Emotions		& 0.1824		& 0.1555$\pm$0.0052       & 0.1553$\pm$0.0045     & 0.1516$\pm$0.0025     & 0.1524$\pm$0.0057     & 0.1524$\pm$0.0037     & 0.1424$\pm$0.0039     & 0.1497$\pm$0.0042     & \textbf{0.1424$\pm$0.0037}	\\
            Enron		    & 0.0150		& 0.0144$\pm$0.0004       & 0.0143$\pm$0.0004     & 0.0150$\pm$0.0003     & 0.0144$\pm$0.0004     & 0.0143$\pm$0.0004     & 0.0143$\pm$0.0003     & 0.0143$\pm$0.0004     & \textbf{0.0142$\pm$0.0004}	\\
            Image		    & 0.1579		& 0.1516$\pm$0.0049       & 0.1456$\pm$0.0038     & 0.1472$\pm$0.0036     & 0.1539$\pm$0.0054     & 0.1456$\pm$0.0046     & 0.1446$\pm$0.0042     & 0.1481$\pm$0.0045     & \textbf{0.1408$\pm$0.0028}	\\
            Mediamill		& 0.0071		& 0.0070$\pm$0.0000       & 0.0070$\pm$0.0000     & 0.0070$\pm$0.0000     & 0.0071$\pm$0.0000     & 0.0070$\pm$0.0000     & 0.0070$\pm$0.0000     & 0.0070$\pm$0.0000     & \textbf{0.0070$\pm$0.0000}	\\
            Medical		    & 0.0138		& 0.0116$\pm$0.0004       & 0.0076$\pm$0.0002     & 0.0080$\pm$0.0002     & 0.0082$\pm$0.0003     & 0.0076$\pm$0.0002     & 0.0071$\pm$0.0002     & 0.0116$\pm$0.0002     & \textbf{0.0067$\pm$0.0000}	\\
            Number		    & 0.1050		& 0.1026$\pm$0.0037       & 0.1002$\pm$0.0026     & 0.0953$\pm$0.0025     & 0.0975$\pm$0.0028     & 0.1002$\pm$0.0035     & \textbf{0.0948$\pm$0.0026}     & 0.1067$\pm$0.0031     & 0.0953$\pm$0.0025	\\
            Recreation	    & 0.0391		& 0.0419$\pm$0.0014       & 0.0375$\pm$0.0009     & 0.0381$\pm$0.0010     & 0.0373$\pm$0.0005     & 0.0375$\pm$0.0013     & 0.0368$\pm$0.0010     & 0.0396$\pm$0.0010     & \textbf{0.0367$\pm$0.0009}	\\
            Scene		    & 0.0773		& 0.0836$\pm$0.0032       & 0.0763$\pm$0.0015     & 0.0824$\pm$0.0019     & 0.0788$\pm$0.0018     & 0.0763$\pm$0.0020     & 0.0775$\pm$0.0018     & 0.0785$\pm$0.0017     & \textbf{0.0756$\pm$0.0014}	\\
            Science		    & 0.0215		& 0.0233$\pm$0.0006       & 0.0212$\pm$0.0003     & 0.0218$\pm$0.0006     & 0.0214$\pm$0.0004     & 0.0212$\pm$0.0005     & 0.0211$\pm$0.0005     & 0.0227$\pm$0.0005     & \textbf{0.0207$\pm$0.0005}	\\
            Yeast		    & 0.0645		& 0.0634$\pm$0.0022       & 0.0631$\pm$0.0020     & 0.0638$\pm$0.0016     & 0.0637$\pm$0.0017     & 0.0631$\pm$0.0018     & 0.0632$\pm$0.0015     & 0.0639$\pm$0.0017     & \textbf{0.0630$\pm$0.0015}	\\
            \hline
            \multicolumn{10}{l}{The smaller the value, the better the performs}
          \end{tabular}
    }
\end{table}
From these tables, it can be seen that MSFS has the best performance 
for 35 times, ranking first, 
E-net (using $\ell_{1}$ and $\ell_{2}$ regularization at the same time) 
12 times, ranking second, 
MSSL (using manifold regularization) 
and RFS (only using $\ell_{2,1}$ regularization) 2 times, ranking third. 
We can get the following several conclusions. 
First of all, even if E-net only uses joint sparse regularization, 
it is far superior to other methods, 
which shows that joint sparse regularization method has better 
feature selection ability than other methods. 
Secondly, when only $\ell_{2,1}$ regularization is used, 
RFS method does not perform well, 
that is, $\ell_{2,1}$ regularization alone does not have 
strong feature selection ability. 
Then, after introducing manifold regularization 
on the basis of $\ell_{2,1}$ regularization, 
MSSL method is superior to RFS, which shows that 
manifold regularization can improve sparse regularization.
Next, the proposed method is superior to E-net method 
because it uses manifold regularization to obtain 
more accurate spatial structure. 
Compared with the MSSL method, 
we use the joint sparse regularization method 
to obtain more discriminant features.
Finally, although various neural network models have excellent performance 
in feature extraction and prediction, 
their ability is not outstanding in the field of feature selection, 
which is due to it lack of interpretation of original features to some extent.

\begin{table}[htbp]
  \centering
  \caption{One Error(mean$\pm$std) of 8 different algorithms on 11 different datasets} \label{table4}
  \resizebox{\textwidth}{!}{
      \begin{tabular}{llllllllll}
          \hline
          Datasets  & Base    & MSSL    & CSFS    & RFS     & LASSO   & Ridge   & E-net   & DeepFS   & MSFS        \\ 
          \hline
          Arts		    & 0.6160		& 0.6170$\pm$0.0187       & 0.5888$\pm$0.0185     & 0.6092$\pm$0.0170     & 0.5801$\pm$0.0190     & 0.5888$\pm$0.0247     & \textbf{0.5731$\pm$0.0144}     & 0.6347$\pm$0.0223     & 0.5767$\pm$0.0134		\\
          Emotions		& 0.4257		& \textbf{0.3218$\pm$0.0095}       & 0.3416$\pm$0.007      & 0.3465$\pm$0.0074     & 0.3465$\pm$0.0105     & 0.3416$\pm$0.0096     & 0.3366$\pm$0.0072     & 0.3366$\pm$0.0111     & 0.3366$\pm$0.0075		\\
          Enron		    & 0.3276        & 0.2790$\pm$0.0094       & 0.2530$\pm$0.0067     & 0.3276$\pm$0.0058     & 0.2669$\pm$0.0056     & 0.2530$\pm$0.0093     & 0.2530$\pm$0.0054     & 0.2686$\pm$0.0061     & \textbf{0.2444$\pm$0.0071}		\\
          Image		    & 0.4300		& 0.3940$\pm$0.0144       & 0.3860$\pm$0.0079     & 0.3840$\pm$0.0094     & 0.4080$\pm$0.0122     & 0.3860$\pm$0.0146     & 0.3860$\pm$0.0101     & 0.3940$\pm$0.0120     & \textbf{0.3750$\pm$0.0069}		\\
          Mediamill		& 0.3003		& 0.2708$\pm$0.0094       & 0.2688$\pm$0.0076     & 0.2668$\pm$0.0057     & 0.2737$\pm$0.0096     & 0.2688$\pm$0.0097     & 0.2657$\pm$0.0069     & 0.2713$\pm$0.0093     & \textbf{0.2639$\pm$0.0074}		\\
          Medical		    & 0.4667		& 0.4124$\pm$0.0177       & 0.2372$\pm$0.0051     & 0.2651$\pm$0.0085     & 0.2744$\pm$0.0092     & 0.2372$\pm$0.0064     & 0.2202$\pm$0.0051     & 0.3984$\pm$0.0089     & \textbf{0.2109$\pm$0.0059}		\\
          Number		    & 0.2410		& 0.1910$\pm$0.0050       & 0.2280$\pm$0.0069     & 0.2040$\pm$0.0061     & \textbf{0.1270$\pm$0.0025}     & 0.2280$\pm$0.0066     & 0.1270$\pm$0.0026     & 0.2410$\pm$0.0071     & 0.1990$\pm$0.0042		\\
          Recreation	    & 0.6316		& 0.6934$\pm$0.0244       & 0.5897$\pm$0.0121     & 0.6014$\pm$0.0123     & 0.5930$\pm$0.0180     & 0.5897$\pm$0.0158     & 0.5769$\pm$0.0132     & 0.6300$\pm$0.0165     & \textbf{0.5766$\pm$0.0157}		\\
          Scene		    & 0.3094		& 0.3294$\pm$0.0122       & 0.3060$\pm$0.0102     & 0.3135$\pm$0.0075     & 0.2993$\pm$0.0117     & 0.3060$\pm$0.0115     & 0.2826$\pm$0.0051     & 0.3052$\pm$0.0107     & \textbf{0.2793$\pm$0.0075}		\\
          Science		    & 0.7748		& 0.8436$\pm$0.0192       & 0.7562$\pm$0.0173     & 0.7802$\pm$0.0182     & 0.7685$\pm$0.0241     & 0.7562$\pm$0.0270     & 0.7525$\pm$0.0169     & 0.8175$\pm$0.0252     & \textbf{0.7472$\pm$0.0119}		\\
          Yeast		    & 0.2868		& \textbf{0.2639$\pm$0.0069}       & 0.2824$\pm$0.0072     & 0.2726$\pm$0.0056     & 0.2792$\pm$0.0091     & 0.2726$\pm$0.0098     & 0.2770$\pm$0.0073     & 0.2715$\pm$0.0081     & 0.2650$\pm$0.0042		\\
          \hline
          \multicolumn{10}{l}{The smaller the value, the better the performs}
      \end{tabular}
  }
\end{table}

\begin{table}[htbp]
  \centering
  \caption{Coverage(mean$\pm$std) of 8 different algorithms on 11 different datasets} \label{table5}
  \resizebox{\textwidth}{!}{
      \begin{tabular}{llllllllll}
          \hline
          Datasets  & Base    & MSSL    & CSFS    & RFS     & LASSO   & Ridge   & E-net   & DeepFS   & MSFS        \\ 
          \hline
          Arts		    & 15.6163		& 15.7259$\pm$0.5128      & 14.9189$\pm$0.3834    & 15.3060$\pm$0.4891    & 14.5818$\pm$0.4101    & 14.9189$\pm$0.5891    & \textbf{14.2577$\pm$0.2842}    & 15.7081$\pm$0.5573     & 14.5375$\pm$0.3391		\\
          Emotions		& 3.0050		& 2.5594$\pm$0.0664       & 2.5396$\pm$0.0752     & 2.5347$\pm$0.0682     & 2.5248$\pm$0.0820     & 2.5248$\pm$0.0911     & 2.4554$\pm$0.0528     & 2.5842$\pm$0.0967     & \textbf{2.4505$\pm$0.0572}		\\
          Enron		    & 24.6118		& 23.5511$\pm$0.6208      & 23.0485$\pm$0.5181    & 24.6118$\pm$0.5306    & 22.9792$\pm$0.6218    & 23.0485$\pm$0.7633    & \textbf{22.7106$\pm$0.5541}    & 23.1057$\pm$0.7528     & 22.7106$\pm$0.6515		\\
          Image		    & 1.4060		& 1.3910$\pm$0.0467       & 1.3470$\pm$0.0359     & 1.3570$\pm$0.0276     & 1.3940$\pm$0.0454     & 1.3470$\pm$0.0458     & 1.3400$\pm$0.0326      & 1.3690$\pm$0.0501     & \textbf{1.2960$\pm$0.0355}		\\
          Mediamill		& 70.2928		& 63.6074$\pm$2.0855      & 64.1155$\pm$1.7667    & 64.2577$\pm$1.5544    & 62.9520$\pm$1.2791    & 64.1155$\pm$1.902     & 62.9520$\pm$1.6103     & 63.6138$\pm$1.8578    & \textbf{62.6714$\pm$1.5689}		\\
          Medical		    & 16.9302		& 14.4264$\pm$0.4000      & 9.0791$\pm$0.1957     & 9.7659$\pm$0.2945     & 9.6899$\pm$0.2941     & 9.0791$\pm$0.2437     & \textbf{7.9609$\pm$0.2066}     & 13.8682$\pm$0.3126     & 8.0465$\pm$0.1872		\\
          Number		    & 5.6820		& 5.7080$\pm$0.2242       & 5.9130$\pm$0.1587     & 5.5910$\pm$0.1196     & 5.3460$\pm$0.1401     & 5.9130$\pm$0.2253     & \textbf{5.3460$\pm$0.1327}     & 6.0330$\pm$0.2094      & 5.5920$\pm$0.1305		\\
          Recreation	    & 10.5960		& 11.1007$\pm$0.4260      & 10.5685$\pm$0.2963    & 10.7797$\pm$0.2810    & 10.5010$\pm$0.3758    & 10.5685$\pm$0.3901    & \textbf{10.3472$\pm$0.2946}    & 10.8449$\pm$0.3536     & 10.4503$\pm$0.2893		\\
          Scene		    & 0.9264		& 0.9331$\pm$0.0353       & 0.9373$\pm$0.0273     & 0.9816$\pm$0.0252     & 0.8946$\pm$0.0278     & 0.9373$\pm$0.0316     & 0.8654$\pm$0.0156     & 0.8896$\pm$0.0337     & \textbf{0.8620$\pm$0.0148}		\\
          Science		    & 21.8863		& 23.0467$\pm$0.7623      & 21.5680$\pm$0.5938    & 22.3369$\pm$0.4698    & 21.6064$\pm$0.7420    & 21.5680$\pm$0.6806    & 21.6064$\pm$0.5425    & 22.6348$\pm$0.6258    & \textbf{21.4239$\pm$0.4758}		\\
          Yeast		    & 9.1897		& 8.7219$\pm$0.3227       & 8.7917$\pm$0.2827     & 8.6565$\pm$0.1905     & 8.7852$\pm$0.2423     & 8.7917$\pm$0.2744     & 8.7634$\pm$0.1849     & 8.6917$\pm$0.2659     & \textbf{8.6401$\pm$0.1575}		\\
          \hline
          \multicolumn{10}{l}{The smaller the value, the better the performs}
      \end{tabular}
  }
\end{table}

\begin{table}[htbp]
  \centering
  \caption{Average Precision(mean$\pm$std) of 8 different algorithms on 11 different datasets} \label{table6}
  \resizebox{\textwidth}{!}{
      \begin{tabular}{llllllllll}
          \hline
          Datasets  & Base    & MSSL    & CSFS    & RFS     & LASSO   & Ridge   & E-net   & DeepFS   & MSFS        \\ 
          \hline
          Arts		    & 0.2418		& 0.2303$\pm$0.0084       & 0.2634$\pm$0.0085     & 0.2418$\pm$0.0055     & 0.2782$\pm$0.0100     & 0.2634$\pm$0.0114     & \textbf{0.2942$\pm$0.0068}     & 0.2177$\pm$0.0079     & 0.2854$\pm$0.0068		\\
          Emotions		& 0.5381		& 0.6159$\pm$0.0121       & 0.6155$\pm$0.0153     & 0.6276$\pm$0.0153     & 0.6344$\pm$0.0171     & 0.6344$\pm$0.0226     & 0.6572$\pm$0.0187     & 0.6159$\pm$0.0181     & \textbf{0.6572$\pm$0.0150}		\\
          Enron		    & 0.3850		& 0.4054$\pm$0.0145       & 0.4381$\pm$0.0117     & 0.385$\pm$0.0103      & 0.4303$\pm$0.0096     & 0.4381$\pm$0.0163     & \textbf{0.4424$\pm$0.0073}     & 0.4302$\pm$0.0124     & 0.4396$\pm$0.0081		\\
          Image		    & 0.5897		& 0.6093$\pm$0.0159       & 0.6219$\pm$0.0224     & 0.6152$\pm$0.0135     & 0.5945$\pm$0.0161     & 0.6219$\pm$0.0157     & 0.6272$\pm$0.0158     & 0.6090$\pm$0.0160     & \textbf{0.6362$\pm$0.0130}		\\
          Mediamill		& 0.3930		& 0.4143$\pm$0.0133       & 0.4116$\pm$0.0138     & 0.4124$\pm$0.0115     & 0.4107$\pm$0.0124     & 0.4116$\pm$0.0139     & \textbf{0.4155$\pm$0.0088}     & 0.4129$\pm$0.0095     & 0.4140$\pm$0.0084		\\
          Medical		    & 0.4328		& 0.5183$\pm$0.0187       & 0.6961$\pm$0.0178     & 0.6775$\pm$0.0208     & 0.6711$\pm$0.0216     & 0.6961$\pm$0.0223     & 0.7163$\pm$0.0193     & 0.5155$\pm$0.0237     & \textbf{0.7321$\pm$0.0211}		\\
          Number		    & 0.6267		& 0.6079$\pm$0.0243       & 0.6314$\pm$0.0211     & \textbf{0.6678$\pm$0.0165}     & 0.6595$\pm$0.0197     & 0.6314$\pm$0.0167     & 0.6642$\pm$0.0194     & 0.5825$\pm$0.0223     & 0.6619$\pm$0.0169		\\
          Recreation	    & 0.2555		& 0.1997$\pm$0.0086       & 0.2879$\pm$0.0105     & 0.2769$\pm$0.0071     & 0.2912$\pm$0.0091     & 0.2879$\pm$0.0082     & 0.3022$\pm$0.0068     & 0.2472$\pm$0.0102     & \textbf{0.3034$\pm$0.0070}		\\
          Scene		    & 0.7027		& 0.6800$\pm$0.0215       & 0.7089$\pm$0.0170     & 0.6897$\pm$0.0161     & 0.7023$\pm$0.0251     & 0.7089$\pm$0.0240     & 0.7148$\pm$0.0170     & 0.7013$\pm$0.0262     & \textbf{0.7189$\pm$0.0136}		\\
          Science		    & 0.2018		& 0.1352$\pm$0.0043       & 0.2150$\pm$0.0070     & 0.1931$\pm$0.0043     & 0.2074$\pm$0.0073     & 0.2150$\pm$0.0082     & 0.2181$\pm$0.0039     & 0.1573$\pm$0.0069     & \textbf{0.2327$\pm$0.0042}		\\
          Yeast		    & 0.6147		& \textbf{0.6252$\pm$0.0148}       & 0.6219$\pm$0.0152     & 0.6229$\pm$0.0124     & 0.6204$\pm$0.0227     & 0.6219$\pm$0.0197     & 0.6227$\pm$0.0186     & 0.6204$\pm$0.0182     & 0.6238$\pm$0.0151		\\
          \hline
          \multicolumn{10}{l}{The greater the value, the better the performs}
      \end{tabular}
  }
\end{table}
In fact, we can use more label information in multi-label learning methods, 
instead of just calculating loss. 
When we construct the neighborhood graph, 
we integrate the information of features and labels, 
which makes the neighbors of samples more credible 
and the local structure more accurate.

To describe the performance of different methods 
when selecting different numbers of features on different data sets, 
these results are visualized. For details, please refer to Fig. 1 to Fig. 5. 
It can be seen that the proposed methods are 
among the best in almost all data sets and evaluation metrics.

Although various methods have been widely used, 
there are some defects at present. 
For example, in Fig.2 (e), 
the performance of each method does not improve with 
the increase of the number of selected features, 
which is mainly due to the imbalance of labels. 
The model will focus on predicting the primary labels, 
thus ignoring the secondary labels. 
In extreme cases, when the instances are concentrated on a few labels, 
the model will predict all the values of these labels as 1, 
while all the other labels are predicted as 0. 
At this time, no matter how many features are increased, 
their performance cannot be improved.

\subsection{Sensitivity analysis}
There are several parameters in the proposed method, 
such as $\alpha$, $\beta$, etc. 
We designed experiments to verify 
how these parameters affect the performance of the model.
Figs. 6 analyze the sensitivity of Average precision to 
parameters $\alpha$ and $\beta$ by contour map. 
The darker the color, the greater the value.
In Fig. 7, we fixed the parameter of $\alpha$ and $\beta$ as $0.1$ 
and $10$ respectively. 
On these data sets, 
it is obvious that the contour lines have shown a non-strip distribution, 
which is caused by selecting different $\rho$ values. 
The result verifies the necessity of introducing joint sparse regularization.

In addition, we have one additional parameters, the number of random walks. 
In fact, when we set the number of random walk steps between 
\{30, 40, 50, 60, 70, 80\}, 
the performance of the model is basically in a stable state, 
and the evaluation metrics in various aspects will not fluctuate greatly. 
This shows that the random walk strategy is robust and superior 
to other methods.

\subsection{Significance test}

In order to test whether the proposed method is 
statistically significantly superior to other methods, 
we performed bonferroni-dunn test, and the results are shown in Fig. 8. 
The results show that when the significance level is set as 0.05, 
MSFS is significantly superior to almost all other methods used 
for comparison except E-net method. 
Even though the statistical difference between MSFS 
and E-net is not significant, 
MSFS always performs better on the 11 data sets we experimented with.

\section*{Conclusions\label{sec6}}

In this paper, we construct a new feature selection model 
in the context of multi-label learning, 
which introduces both manifold regularization and joint sparse regularization. 
In the manifold regularization constraint, 
we combine the information of features and labels 
to establish a joint similarity matrix, 
which measures the distance between instances better 
and effectively avoids the occurrence of the "short circuit". 
By normalizing the rows of the joint similarity matrix, 
it has the properties of a random matrix. 
Therefore, we can implement a random walk strategy to further 
extract highly robust local structures. 
It is verified by experiments that the model can extract a stable 
and effcient structure with fewer transfer steps. 
To make the coeffcient matrix sparse, 
we balance the $\ell_{F}$ regularization and the $\ell_{2,1}$ regularization, 
so that the model can comprehensively consider 
the situation where multivariable work together. 
Compared with the previous methods, 
the proposed method has higher accuracy and lower loss 
due to introduce the joint similarity matrix 
and the joint sparse constraints. 
However, this method has a shortcoming, that is, 
the complexity of the algorithm is relatively higher 
due to the need to implement the state transition process. 
If the data set is very large, then the algorithm will 
have higher requirements on the equipment, 
and the effciency of the algorithm will be much lower.

In real world, data labels are often unavailable or expensive to get, 
which makes supervised learning hard to achieve. 
Therefore, our next work mainly focuses on feature selection 
in the field of semi-supervised learning. 
What's more, the label imbalance usually seriously affects 
the accuracy of the inference results, 
so we will also consider the problem of label imbalance.

\section*{Declarations\label{sec:7}}

This research did not receive any specific grant from funding agencies 
in the public, commercial, or not-for-profit sectors.

\bibliographystyle{myjmva}
\bibliography{trial}
\end{document}